\documentclass[11pt]{article}

\usepackage[T1]{fontenc}

\usepackage[letterpaper, left=1in, right=1in, top=1in, bottom=1in]{geometry}

\usepackage[dvipsnames]{xcolor}

\usepackage{algorithm}

\usepackage{natbib}
\bibpunct{(}{)}{;}{a}{,}{,}
\usepackage[colorlinks=true, linkcolor=blue, citecolor=blue,pagebackref]{hyperref}

\usepackage{amsthm}
\usepackage{mathtools}
\usepackage{amsmath}
\usepackage{amsfonts}
\usepackage{amssymb}

\theoremstyle{definition}  

\newtheorem{lemma}{Lemma}

\newtheorem{definition}{Definition}
\theoremstyle{plain}

\newtheorem{theorem}{Theorem}

\usepackage{xpatch}

\usepackage[dvipsnames]{xcolor}

\newcommand{\ind}[1]{\ensuremath{{\mathbf I}{\left\{#1\right\}}}}
\def\deq{\triangleq}

\newcommand{\En}{\mathbb E}

\newcommand{\inner}[1]{\left\langle #1 \right\rangle}
\newcommand{\reals}{{\mathbb R}}

\newcommand{\norm}[1]{\left\|#1\right\|}

\title{Does data interpolation contradict statistical optimality?}

\date{}

\author{Mikhail Belkin \\ The Ohio State University  \and Alexander Rakhlin \\ MIT \and Alexandre B. Tsybakov \\ CREST, ENSAE}

\begin{document}
	
\maketitle
\begin{abstract}
	We show that learning methods interpolating the training data can achieve optimal rates for the problems of nonparametric regression and prediction with square loss. 
	\end{abstract}

\section{Introduction}

In this paper, we exhibit estimators that interpolate the data, yet achieve optimal rates of convergence for the problems of nonparametric regression and prediction with square loss. This curious observation goes against the usual (or, folklore?) intuition that a good statistical procedure should forego the exact fit to data in favor of a more smooth representation. The family of estimators we consider do exhibit a bias-variance trade-off with a tuning parameter, yet this ``regularization'' co-exists in harmony with data interpolation.

Motivation for this work is the recent focus within the machine learning community on the out-of-sample performance of neural networks. These flexible models are typically trained to fit the data exactly (either in their sign or in the actual value), yet they predict well on unseen data. The conundrum has served both as a source of excitement about the ``magical'' properties of neural networks, as well as a call for the development of novel statistical techniques to resolve it. 

The aim of this short note is to show that not only can interpolation be a good statistical procedure, but it can even be optimal in a minimax sense. To the best of our knowledge, such optimality has not been exhibited before.

Let $(X,Y)$ be a random pair on $\reals^d\times \reals$ with distribution $P_{XY}$, and let $f(x) = \En[Y|X=x]$ be the regression function. A goal of nonparametric estimation is to construct an estimate $f_n$ of $f$, given a sample $(X_1,Y_1),\ldots,(X_n,Y_n)$ drawn independently from $P_{XY}$. A classical approach to this problem is kernel smoothing. In particular, the Nadaraya-Watson estimator \citep{nadaraya1964estimating,watson1964smooth} is defined as
\begin{align}
	\label{eq:NW_def}
	f_n(x) = \frac{\sum_{i=1}^n Y_i K\left(\frac{x-X_i}{h}\right)}{\sum_{i=1}^n K\left(\frac{x-X_i}{h}\right)},
\end{align}
where $K:\reals^d\to\reals$ is a kernel function and $h>0$ is a bandwidth and we assume that the denominator does not vanish. Appropriate choices of $K$ and $h$ lead to optimal rates of estimation, under various assumptions, and we refer the reader to \citep{tsybakov2009introduction} and references therein.

We consider singular kernels that approach infinity when their argument tends to zero. It has been observed, at least since \citep{shepard1968two}, that the resulting function in \eqref{eq:NW_def} interpolates the data. We will focus on the particular kernel
\begin{align}
	\label{eq:def_our_kernel}
	K\left(u\right) \deq \norm{u}^{-a} \ind{\norm{u}\leq 1}, 
\end{align}
for some $a>0$. Here, $\norm{\cdot}$ denotes the Euclidean norm. Our results can be extended to other related singular kernels, for example, to 
\begin{align}
	\label{eq:def_katkovnik_kernel}
	K\left(u\right) \deq \norm{u}^{-a} [1-\norm{u}]^2_+ 
\end{align}
where $[c]_+ = \max\{c,0\}$, and
\begin{align}
	\label{eq:def_cap_kernel}
	K\left(u\right) \deq \norm{u}^{-a} \cos^2(\pi \norm{u}/2) \ind{\norm{u}\leq 1},
\end{align}
considered in \citep{lancaster1981surfaces,katkovnik1985book}. Also,  $\norm{\cdot}$ can be any norm on $\reals^d$, not necessarily the Euclidean norm.

Our main result, stated precisely in the next section and proved in Section~\ref{sec:proofs}, establishes that
\begin{align*}
	\En\norm{f_n-f}^2_{L_2(P_X)} \deq \En(f_n(X)-f(X))^2 \leq C n^{-\frac{2\beta}{2\beta+d}}
\end{align*}
whenever the regression function $f$ belongs to a H\"older class with parameter $\beta\in(0,2]$, and under additional assumptions stated below. Here $C$ is a constant that does not depend on $n$ and $P_X$ is the marginal distribution of $X$. The rate $n^{-\frac{2\beta}{2\beta+d}}$ is the classical minimax optimal rate for these classes.

Our result also yields a curious conclusion for the problem of prediction with square loss. Observe that excess loss---an object studied in Statistical Learning Theory---with respect to a H\"older class $\Sigma(\beta, L)$, formally defined below, can be written as
\begin{align*}
	&\En (f_n(X)-Y)^2 - \inf_{g\in \Sigma(\beta,L)} \En(g(X)-Y)^2 \\
	&= \En (f_n(X)-f(X))^2 - \inf_{g\in \Sigma(\beta,L)} \En(g(X)-f(X))^2 \\
	&= \En (f_n(X)-f(X))^2
\end{align*}
under the assumption that the model is \emph{well-specified} (that is, the regression function is in the class). We remark that the estimator $f_n$ is \emph{improper}, in the sense that it does not itself belong to the H\"older class (its smoothness depends on $h$ and, hence, on $n$). In conclusion, despite the fact that $f_n$ is improper and fits the data exactly, it attains optimal rates for excess loss. We refer the reader to  \citep{rakhlin2017empirical} for further discussion of optimal rates in nonparametric estimation and statistical learning.

\paragraph{Prior work}

Within the context of pattern classification, the 1-Nearest-Neighbor classifier is an example of an interpolating rule. It is shown in \citep{cover1967nearest} that the limit (as $n$ tends to infinity) of the classification risk is no more than twice the Bayes risk. To make $k$-Nearest-Neighbor rules consistent, one is required to increase $k$ with $n$ \citep{devroye1996probabilistic,chaudhuri2014rates}, in which case the rule is no longer interpolating.

The idea of interpolating the data using singular kernels appears already in \citep{shepard1968two} and was further developed in \citep{lancaster1981surfaces,katkovnik1985book}, among others. These works were focusing on deterministic properties of the interpolants and no statistical guarantees have been established until \cite{devroye1998hilbert} have shown consistency of the estimator \eqref{eq:NW_def} for the singular kernel $K\left(u\right) = \norm{u}^{-d}$, however, without finite sample guarantees. The recent work of \cite{belkin2018overfitting} proves the first (to the best of our knowledge) non-asymptotic rates for interpolating procedures, yet the guarantees are suboptimal. The present paper shows that statistical optimality of interpolating rules can indeed be achieved and it holds under rather standard nonparametric assumptions on the regression function.

\section{Main Results}
\label{sec:main}

We start with a definition.
\begin{definition}
	For $L>0$ and $\beta\in (0,2]$, the $(\beta,L)$-H\"older class, denoted by $\Sigma(\beta,L)$, is defined as follows:
	\begin{itemize}
		\item If $\beta\in (0,1]$, the class $\Sigma(\beta,L)$ consists of functions $f:\reals^d\to\reals$ satisfying
		\begin{align}
			\label{eq:holder01}
			\forall x,y\in\reals^d,~~~ |f(x)-f(y)|\leq L \norm{x-y}^{\beta}.
		\end{align}
		\item If $\beta\in (1,2]$, the class $\Sigma(\beta,L)$ consists of continuously differentiable functions $f:\reals^d\to\reals$ satisfying
		\begin{align}
			\label{eq:holder12}
			\forall x,y\in\reals^d,~~~ |f(x)-f(y)-\inner{\nabla f(y),x-y}| \leq L \norm{x-y}^\beta
		\end{align}
		where $\inner{\cdot,\cdot}$ denotes the inner product. 
	\end{itemize}
\end{definition}

We assume the following.
\begin{itemize}
	\item[(A1)] For any $x\in\reals^d$, the expectation $\En[Y|X=x]=f(x)$ exists and $\En[\xi^2|X=x]\leq \sigma^2_\xi < \infty$, where $\xi = Y-\En[Y|X] = Y-f(X)$. 
	\item[(A2)] The marginal density $p(\cdot)$ of $X$ exists and satisfies $0<p_{\min} \leq p(x) \leq p_{\max}$ for all $x$ on its support.
\end{itemize}

The Nadaraya-Watson estimator for a singular kernel $K$  is defined as 
\begin{align}
	\label{eq:NW_def_precise}
	f_n(x) = \begin{cases}
		~~ Y_i & \text{if } x=X_i \text{ for some } i=1,\ldots,n, \\
		~~ 0 & \text{if } \sum_{i=1}^n K\left(\frac{x-X_i}{h}\right) = 0, \\
		~~ \frac{\sum_{i=1}^n Y_i K\left(\frac{x-X_i}{h}\right)}{\sum_{i=1}^n K\left(\frac{x-X_i}{h}\right)} & \text{otherwise}.
	\end{cases}
\end{align}

The two main results for this estimator are now stated.
\begin{theorem} 
	\label{thm:beta01}
	Assume that $f\in\Sigma(\beta,L_f)$ for $\beta\in(0,1]$, $L_f>0$. Let Assumptions $(A1)$ and $(A2)$ be satisfied, and $0<a<d/2$. Then for any fixed $x_0\in\reals^d$ in the support of $p$ the estimator \eqref{eq:NW_def_precise} with kernel \eqref{eq:def_our_kernel} and bandwidth $h=n^{-\frac{1}{2\beta+d}}$ satisfies
	$$\En[(f_n(x_0)-f(x_0))^2]\leq C n^{-\frac{2\beta}{2\beta+d}}$$
	where $C>0$ is a constant that does not depend on $n$.
\end{theorem}
\begin{theorem}
	\label{thm:beta12}
	Assume that $f\in\Sigma(\beta,L_f)$ for $\beta\in(1,2]$, $L_f>0$. Let Assumptions $(A1)$ and $(A2)$ be satisfied, and $0<a<d/2$. Assume in addition that $p\in \Sigma(\beta-1,L_p)$, $L_p>0$. Then for any fixed $x_0\in\reals^d$ in the support of $p$ the estimator \eqref{eq:NW_def_precise} with kernel \eqref{eq:def_our_kernel} and bandwidth $h=n^{-\frac{1}{2\beta+d}}$ satisfies $$\En[(f_n(x_0)-f(x_0))^2]\leq C n^{-\frac{2\beta}{2\beta+d}}$$
	where $C>0$ is a constant that does not depend on $n$.	
	\end{theorem}

The pointwise mean squared error (MSE) bounds of Theorems \ref{thm:beta01} and \ref{thm:beta12} immediately imply the integrated MSE with respect to the marginal distribution of $X$:
\begin{align*}
	\En\norm{f_n-f}^2_{L_2(P_X)} \leq C n^{-\frac{2\beta}{2\beta+d}},
\end{align*}
assuming that $f$ is bounded on the support of the marginal density $p$.

\section{Visualization}
\label{sec:visualization}

The figures below show interpolations with kernels \eqref{eq:def_our_kernel} and \eqref{eq:def_katkovnik_kernel}. While both achieve optimal rates of convergence in this simple one-dimensional problem, the latter kernel appears to be less irregular. Indeed, unlike \eqref{eq:def_our_kernel}, kernels \eqref{eq:def_katkovnik_kernel} and \eqref{eq:def_cap_kernel} produce continuous functions.

\begin{figure}[H]
  \centering	
    \includegraphics[width=.65\textwidth]{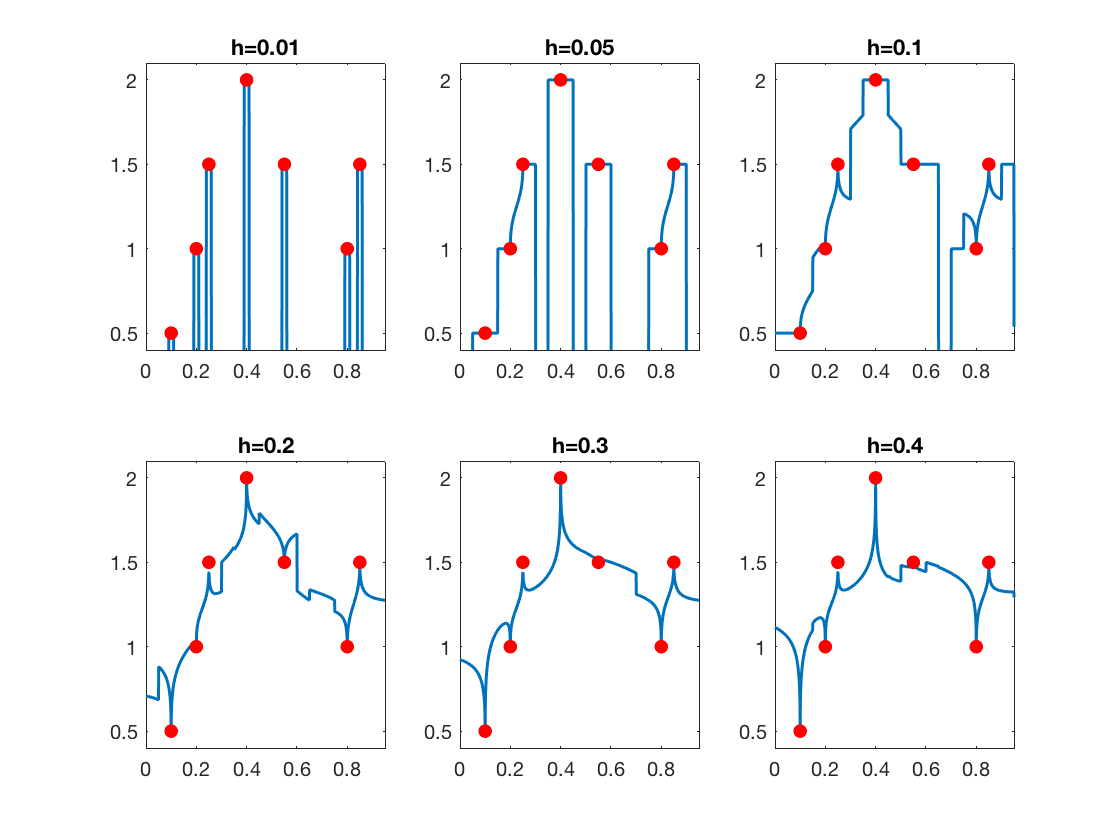}
  \caption{Interpolation with $K\left(u\right) = \norm{u}^{-a} \ind{\norm{u}\leq 1}$, $a=0.49$, and various values of $h$.}
  \label{fig:graphics_i1}
\end{figure}
\begin{figure}[H]
  \centering	
    \includegraphics[width=.65\textwidth]{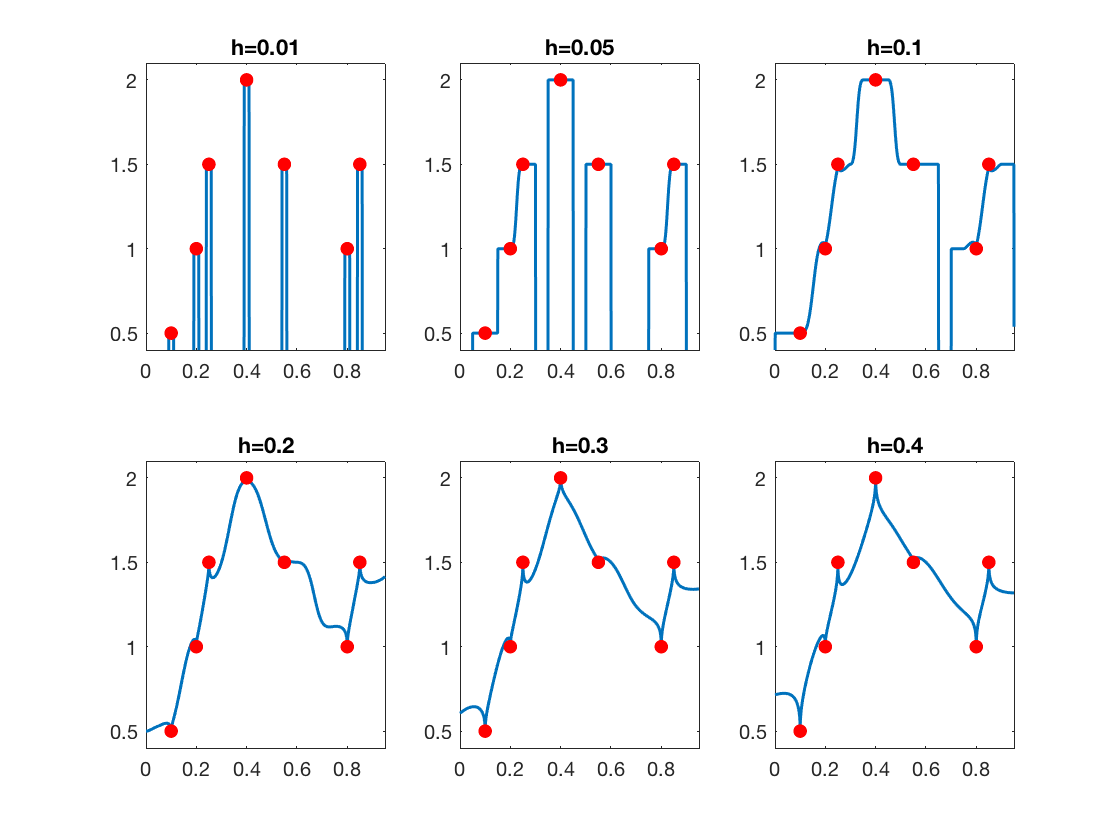}
  \caption{Interpolation with $K\left(u\right) = \norm{u}^{-a} [1-\norm{u}]^2_+$, $a=0.49$, and various values of $h$.}
  \label{fig:graphics_i2}
\end{figure}
We now compare Figures~\ref{fig:graphics_i1} and \ref{fig:graphics_i2} to those with a non-singular kernel. We remark that choices of bandwidth $h$ differ depending on the kernel, and direct comparisons for the same value across kernels might not be meaningful.
\begin{figure}[H]
  \centering	 	\includegraphics[width=.65\textwidth]{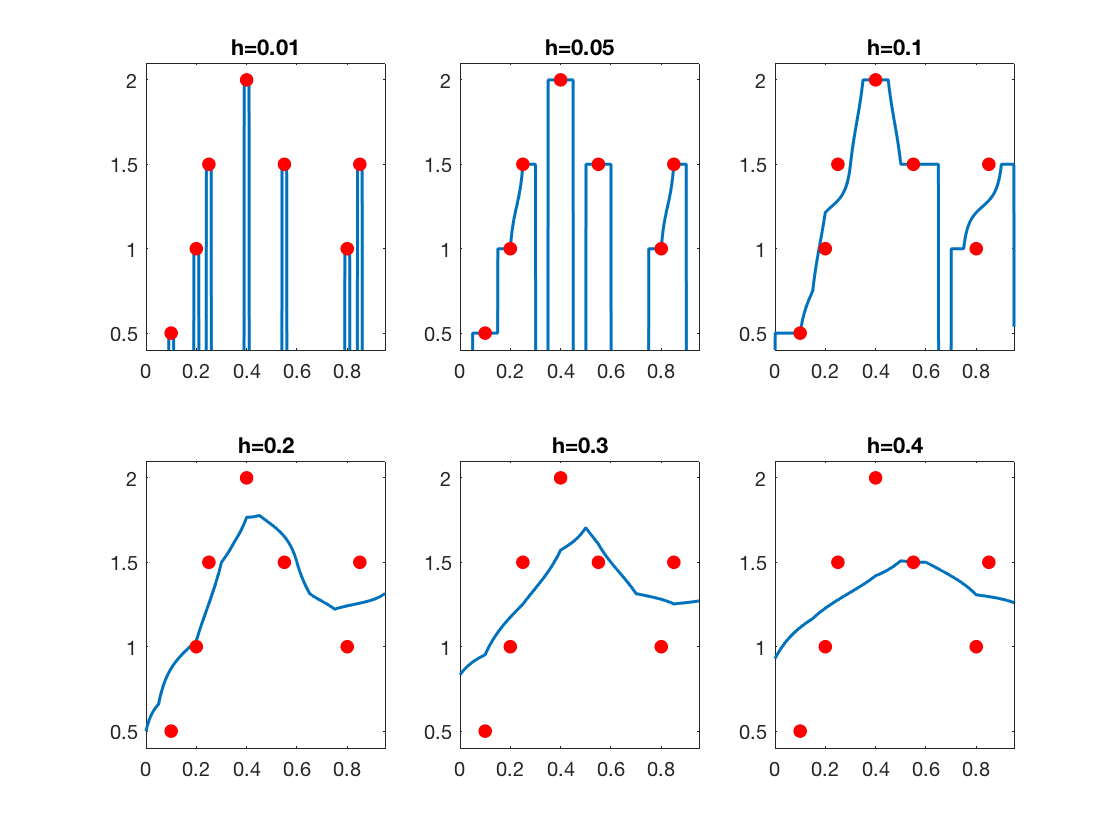}
  \caption{Comparison: non-singular Epanechnikov kernel $K\left(u\right) = (3/4)(1-\norm{u}^2) \ind{\norm{u}\leq 1}$.}
  \label{fig:graphics_epanechnikov}
\end{figure}
\begin{figure}[H]
  \centering	 	\includegraphics[width=.65\textwidth]{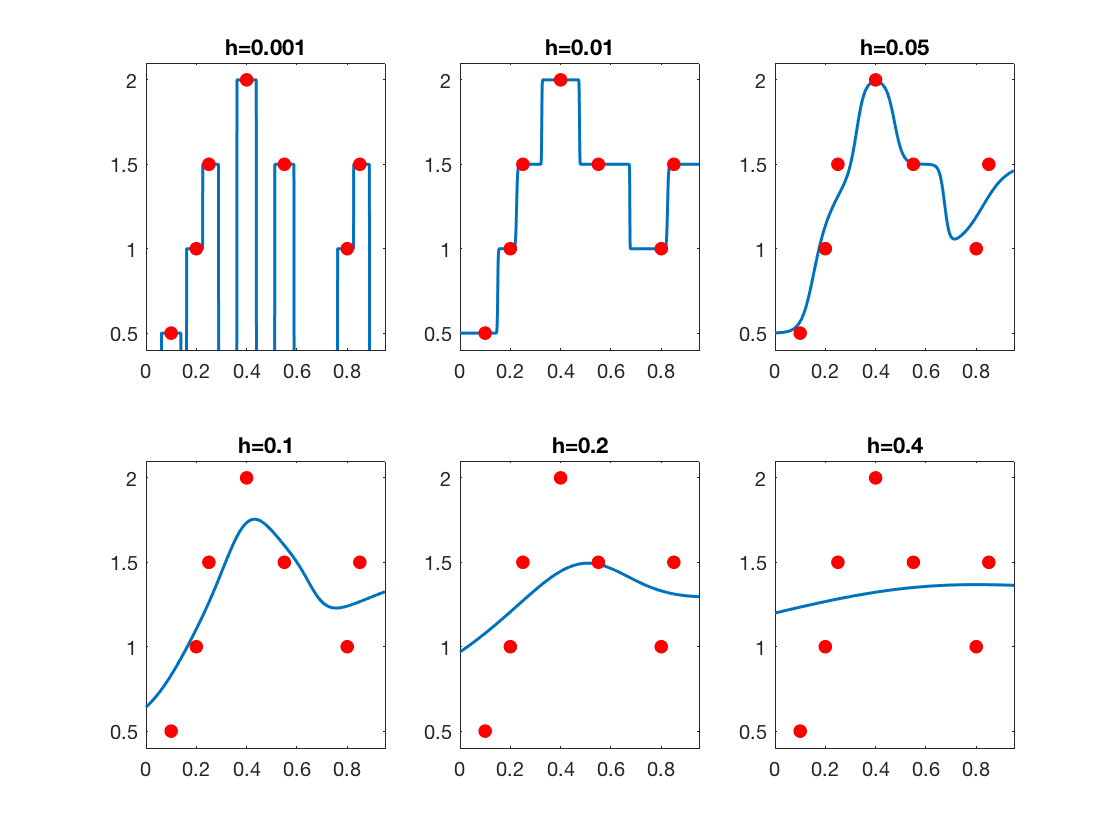}
  \caption{Comparison: non-singular Gaussian kernel $K\left(u\right) = (1/\sqrt{2\pi})\exp\left\{-\norm{u}^2\right\}$. Note the altered choices of $h$.}
  \label{fig:graphics_gaussian}
\end{figure}
Figure~\ref{fig:graphics_i3} below shows a comparison between the interpolating kernel \ref{eq:def_katkovnik_kernel} and the Gaussian kernel for binary-valued data. We observe the more global effect that each point has on the behavior of the solution with the Gaussian kernel, in comparison to the singular kernel. Understanding properties of the plug-in classifier $\text{sign}(f_n)$ under various margin conditions appears to be an interesting direction of further research.
\begin{figure}[H]
  \centering	 	\includegraphics[width=.65\textwidth]{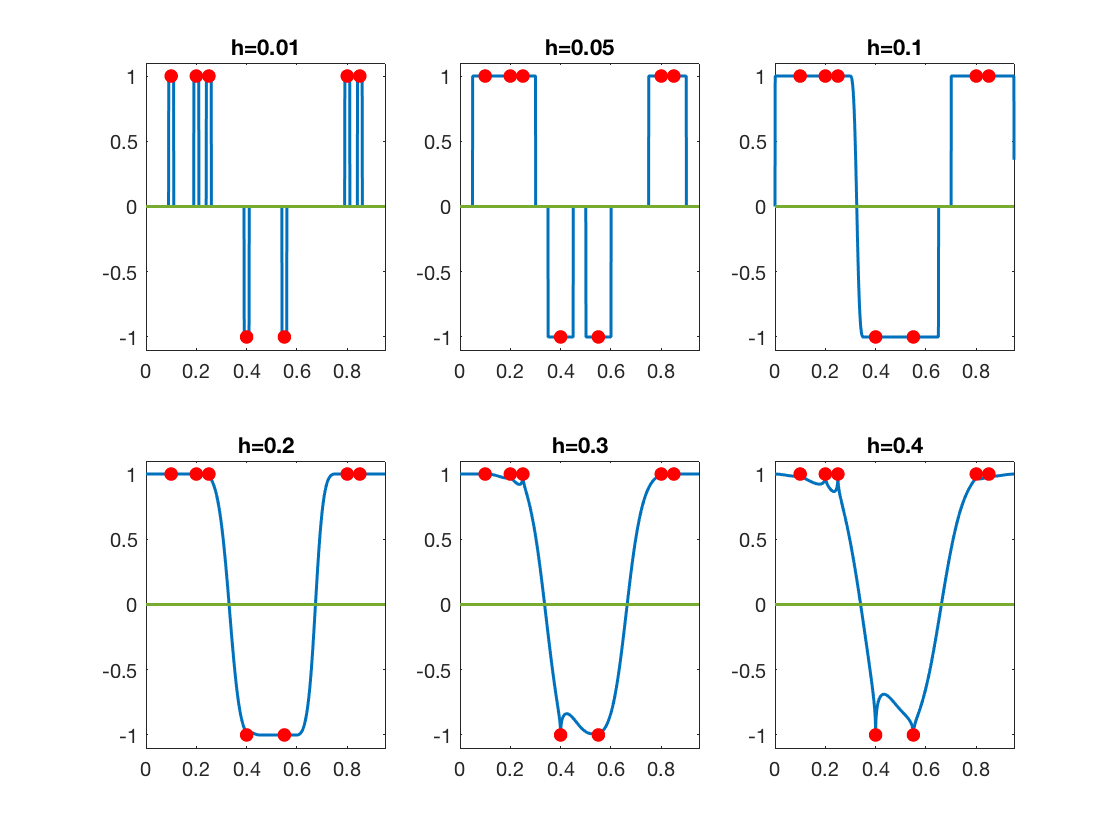}
  \caption{Interpolation with $K\left(u\right) = \norm{u}^{-a} [1-\norm{u}]^2_+$, $a=0.49$, for binary-valued $Y$.}
  \label{fig:graphics_i3}
\end{figure}
\begin{figure}[H]
  \centering	 	\includegraphics[width=.65\textwidth]{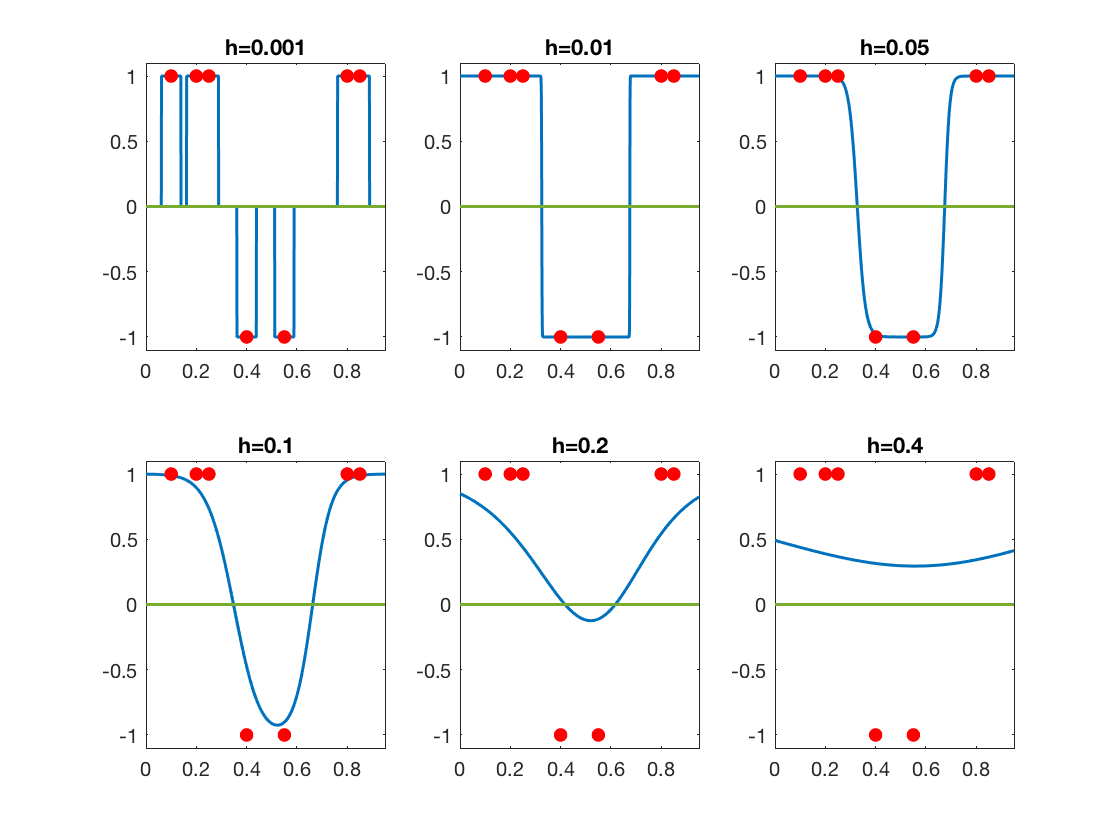}
  \caption{Comparison: non-singular Gaussian kernel $K\left(u\right) = (1/\sqrt{2\pi})\exp\left\{-\norm{u}^2\right\}$ for binary-valued $Y$. Note the altered choices of $h$.}
  \label{fig:graphics_gaussian_binary}
\end{figure}

\section{Proofs}
\label{sec:proofs}
Without loss of generality, consider the problem of estimating $f(x_0)$ at $x_0=0$, assuming it is in the support of $p$ and $|f(x_0)|<\infty$. 

Consider the event 
$$\mathcal{E} = \left\{\sum_{i=1}^n K_h(X_i) \neq 0\right\} = \left\{\exists i=1,\ldots,n: \norm{X_i}\leq h\right\}$$ 
and observe that
$$P\left(\bar{\mathcal{E}}\right)\leq \left(1-Cp_{\min}h^d\right)^n\leq \exp\left\{-Cp_{\min}nh^d\right\}$$
for a constant $C>0$ that does not depend on $n$. On the event $\bar{\mathcal{E}}$, we have $f_n(0)=0$ and thus the contribution to expected risk is at most $M_{\mathcal{E}} = f(0)^2\exp\left\{-Cp_{\min}nh^d\right\},$
a lower-order term compared to the remaining calculations. 

On the event $\mathcal{E}$, the estimator $f_n(0)$ is equal to
\begin{align*}
 	\bar f_n(0) = \frac{\sum_{i=1}^n Y_i K_h(X_i)}{\sum_{i=1}^n K_h(X_i)}
\end{align*}
(modulo an event of zero probability with respect to the joint distribution of $X_1,\ldots,X_n$), where
\begin{align*}
	K_h (x) \deq K(x/h).
\end{align*}
Set $\xi_i=Y_i-f(X_i)$. Let $\En_Y$ denote the expectation with respect to $Y_1,\ldots,Y_n$, conditional on $X_1,\ldots,X_n$. We have the following ``bias-variance'' decomposition
\begin{align*}
	\En[(f_n(0)-f(0))^2] &\le \En[(\bar f_n(0)-\En_Y \bar f_n(0) + \En_Y \bar f_n(0) - f(0))^2\ind{\mathcal{E}}] + M_{\mathcal{E}}\\
	&= \En[(\bar f_n(0)-\En_Y \bar f_n(0))^2 \ind{\mathcal{E}}] + \En[(\En_Y \bar f_n(0) - f(0))^2 \ind{\mathcal{E}}] + M_{\mathcal{E}} ~ .
\end{align*}

It holds that, on the event $\mathcal{E}$,
\begin{align*}
	\En_Y \bar f_n(0) = \frac{\sum_{i=1}^n f(X_i) K_h(X_i)}{\sum_{i=1}^n K_h(X_i)}
\end{align*}
and, hence, the variance term is
\begin{align}
	\label{eq:initial_var_calc}
	\sigma^2(0) \deq \En[(\bar f_n(0)-\En_Y \bar f_n(0))^2\ind{\mathcal{E}}] = \En\left[\left(\frac{\sum_{i=1}^n \xi_i K_h(X_i)}{\sum_{i=1}^n K_h(X_i)}\right)^2 \ind{\mathcal{E}} \right] \leq \sigma_\xi^2 \sigma_X^2 , 
\end{align}
where 
\begin{align*}
	\sigma_X^2 \deq n \En \left[  \frac{K_h^2(X_1)}{\left(\sum_{i=1}^n K_h(X_i)\right)^2} \ind{\mathcal{E}}\right].
\end{align*}
On the other hand, the bias\footnote{To be precise, this term includes variance due to random $X$, as will be clear from Lemma~\ref{lem:bias_holder_12}.} is
\begin{align}
	\label{eq:bias_eq}
	b^2(0) \deq \En[(\En_Y \bar f_n(0) - f(0))^2 \ind{\mathcal{E}} ] &= \En\left[\left(\frac{\sum_{i=1}^n (f(X_i)-f(0)) K_h(X_i)}{\sum_{i=1}^n K_h(X_i)}\right)^2 \ind{\mathcal{E}}  \right].
\end{align}
The following lemmas control each of the above expressions under various assumptions on $f$ and the marginal density $p$. We will denote by $C$ positive constants that can vary from line to line.

\subsection{Bounding the Variance}
\begin{lemma}
	\label{lem:variance}
	Let Assumptions (A1) and (A2) hold. Then,
	\begin{align}
		\sigma^2(0) \leq  \frac{C\sigma_\xi^2}{nh^d}.
	\end{align}
\end{lemma}
\begin{proof}
	Introduce the random variables 
	$$\eta_i = \ind{\norm{X_i}\leq h}.$$
	They are i.i.d. and follow the Bernoulli distribution with parameter
	$$\bar{p} \deq P(\norm{X_1} \leq h)\geq c_0 p_{\min} h^d $$
	where $c_0>0$ depends only on $d$. Then  
	\begin{align}
		\label{eq:two-split-variance}
		\sigma^2_X \le n \En\left[ \frac{ K_h^2(X_1)}{\left(\sum_{i=1}^n K_h(X_i)\right)^2}\, \ind{\sum_{i=1}^n \eta_i \leq \frac{n\bar{p}}{2}}  \ind{\mathcal{E}}\right] +  n \En\left[ \frac{4}{(n\bar{p})^2} K_h^2(X_1) \right]
	\end{align}
	where we have used the fact that
	$$K_h(X_i) \geq \eta_i,~~~ i=1,\ldots,n.$$
	Change of variables yields
	\begin{align}
		\label{eq:square_integrability}
		n\En [K_h^2(X_1) ] \leq nh^d p_{\max}\int_{\reals^d} K^2(u) du.
	\end{align}
	Since the kernel $K$ is radially symmetric and supported on the unit Euclidean ball, the last expression is  bounded from above by
	\begin{align*}
		C nh^d p_{\max} \int_{0}^1 r^{-2a} r^{d-1} dr \leq C_2 nh^d
	\end{align*}
	whenever $d-2a-1> -1$ (equivalently, $a<d/2$). Here $C,C_2$ are positive constants depending only on~$d$. It follows that 
	$$n\En\left[ \frac{4}{(n\bar{p})^2} K_h^2(X_1)\right] \leq \frac{4}{(c_0 p_{\min} nh^d)^2} C_2 nh^d \leq \frac{C}{nh^d}.$$
	To conclude the proof, we analyze the first term in \eqref{eq:two-split-variance}:
	\begin{align*}
		n\En\left[ \frac{K_h^2(X_1)}{\left(\sum_{i=1}^n K_h(X_i)\right)^2} \ \ind{\sum_{i=1}^n \eta_i \leq \frac{n\bar{p}}{2}} \ind{\mathcal{E}} \right] 
		&\leq n P\left( \sum_{i=1}^n \eta_i \leq \frac{n\bar{p}}{2} \right) = n P\left( \sum_{i=1}^n \eta_i - n\bar{p}\leq \frac{n\bar{p}}{2} \right).
	\end{align*}
	By Bernstein's inequality, the last expression is at most
	\begin{align*}
		n\exp\left\{ -\frac{(n\bar{p}/2)^2}{2(n\bar{p}(1-\bar{p}) + n\bar{p}/3)}\right\} \leq n\exp\left\{-\frac{3n\bar{p}}{32}\right\} \leq n\exp\left\{-C nh^d \right\}.
	\end{align*}
\end{proof}

\subsection{Bounding the Bias}

\begin{lemma}
	\label{lem:bias_holder_01}
	Let $\beta\in(0,1]$, $L_f>0$, and assume that $f\in\Sigma(\beta,L_f)$. Then
	$$b^2(0) \leq L_f^2  h^{2\beta}.$$
\end{lemma}
\begin{proof}
	Since $f\in\Sigma(\beta,L_f)$ we have, on the event $\mathcal{E}$,
	\begin{align*}
		\left|\frac{\sum_{i=1}^n (f(X_i)-f(0)) K_h(X_i)}{\sum_{i=1}^n K_h(X_i)}\right| \le 
		\left|\frac{\sum_{i=1}^n L_f\norm{X_i}^\beta K_h(X_i)}{\sum_{i=1}^n K_h(X_i)}\right|
		&\leq L_f h^{\beta}.
	\end{align*}
	The last step holds because the kernel $K_h$ is zero outside of the Euclidean ball of radius $h$. 
\end{proof}

Lemma~\ref{lem:bias_holder_01} can be extended to smoothness $\beta\in(1,2]$ under an additional assumption on the marginal density.
\begin{lemma}
	\label{lem:bias_holder_12}
	Let $\beta\in(1,2]$, $L_f>0$, and $f\in\Sigma(\beta,L_f)$. Assume that  the density $p$ of the marginal distribution of $X$ satisfies $p\in \Sigma(\beta-1,L_p)$, and $p(x)\geq p_{\min}>0$ for all $x$ in the support of $p$. Then
	$$b^2(0) \leq (L_f + \norm{\nabla f(0)} L_p p_{\min}^{-1}) h^{2\beta} + \sigma_X^2.$$
\end{lemma}
\begin{proof}
	We write \eqref{eq:bias_eq} as
	\begin{align*}
		b^2(0) = \En \left[ \sum_{i,j=1}^n G_i G_j \, \ind{\mathcal{E}} \right]
	\end{align*}
	where
	\begin{align*}
		G_i = \frac{(f(X_i)-f(0)) K_h(X_i)}{\sum_{i=1}^n K_h(X_i)}.
	\end{align*}
	For $i\neq j$ we can write 
	\begin{align*}
		\En[G_i G_j\,\ind{\mathcal{E}}] = \En\left[ (f(X_i)-f(0))(f(X_j)-f(0))A(X_i,X_j)\right]
	\end{align*}
	where
	$$A(X_i,X_j) =  \frac{K_h(X_i)K_h(X_j)}{\left(\sum_{i=1}^n K_h(X_i)\right)^2}\,\ind{\mathcal{E}} \geq 0. $$
	We omit for brevity the dependence of $A(X_i,X_j)$ on $(X_k, k\ne i, k\ne j)$. 
	Thus, 
		\begin{align*}
		\En'[G_i G_j\, \ind{\mathcal{E}}] = \int_{\reals^d}\int_{\reals^d} (f(x_i)-f(0))(f(x_j)-f(0))A(x_i,x_j)p(x_i)p(x_j) dx_i dx_j
	\end{align*}
	where $\En'$ denotes the conditional expectation  over $(X_i,X_j)$ for fixed $(X_k, k\ne i, k\ne j)$.
	Let us define
	$$R(x_i) = f(x_i)-f(0) - \inner{\nabla f(0), x_i} ~~\mbox{and}~~ R(x_j) = f(x_j)-f(0) - \inner{\nabla f(0), x_j}.$$
	Then
	\begin{align*}
		\En'[G_i G_j\, \ind{\mathcal{E}}] &= \int_{\reals^d}\int_{\reals^d} \inner{\nabla f(0), x_i}\inner{\nabla f(0), x_j}A(x_i,x_j)p(x_i)p(x_j) dx_i dx_j \\
		&+ 2\int_{\reals^d}\int_{\reals^d} \inner{\nabla f(0), x_i} R(x_j) A(x_i,x_j)p(x_i)p(x_j) dx_i dx_j\\
		&+ \int_{\reals^d}\int_{\reals^d}  R(x_i) R(x_j) A(x_i,x_j)p(x_i)p(x_j) dx_i dx_j
	\end{align*}
	where the factor $2$ arises from symmetry considerations. Now observe that 
	$$\int_{\reals^d} \inner{\nabla f(0), x_i}A(x_i,x_j)p(0)dx_i = 0$$
	for any $x_j$ since the function under the integral is odd for any fixed $(X_k, k\ne i, k\ne j)$. Applying this observation for both $x_i$ and $x_j$ in the first term of the above decomposition, as well as for the second term, we obtain
	\begin{align*}
		\En'[G_i G_j\, \ind{\mathcal{E}}] &= \int_{\reals^d}\int_{\reals^d} \inner{\nabla f(0), x_i}\inner{\nabla f(0), x_j}A(x_i,x_j)(p(x_i)-p(0))(p(x_j)-p(0)) dx_i dx_j \\
		&+ 2 \int_{\reals^d}\int_{\reals^d} \inner{\nabla f(0), x_i} R(x_j) A(x_i,x_j)(p(x_i)-p(0))p(x_j) dx_i dx_j\\
		&+ \int_{\reals^d}\int_{\reals^d}  R(x_i) R(x_j) A(x_i,x_j)p(x_i)p(x_j) dx_i dx_j.
	\end{align*}
	Condition \eqref{eq:holder12} implies that $|R(x_i)|\leq L_f\norm{x_i}^\beta$. Next, recall that $A$ is zero whenever either $\norm{x_i}>h$ or $\norm{x_j}>h$. Using Cauchy-Shwarz inequality for the inner products and the H\"older assumption on $p$, we conclude that
	\begin{align*}
		\En'[G_i G_j\, \ind{\mathcal{E}}] &\leq B^2L_p^2 h^{2\beta}\int_{\reals^d}\int_{\reals^d}  A(x_i,x_j)dx_i dx_j \\
		&+ 2B L_fL_p h^{2\beta} \int_{\reals^d}\int_{\reals^d}  A(x_i,x_j) p(x_j) dx_i dx_j\\
		&+ L_f^2h^{2\beta} \int_{\reals^d}\int_{\reals^d}  A(x_i,x_j)p(x_i)p(x_j) dx_i dx_j
	\end{align*}
	where $B=\norm{\nabla f(0)}^2$. Using the lower bound $p_{\min}$ on the density, completing the square and taking the expectation with respect to $(X_k, k\ne i, k\ne j)$, we establish 
	\begin{align*}
		\En[G_i G_j\, \ind{\mathcal{E}}] &\leq h^{2\beta} (BL_p p_{\min}^{-1} + L_f)^2 \En \left[ \frac{K_h(X_i)K_h(X_j)}{\left(\sum_{i=1}^n K_h(X_i)\right)^2}\, \ind{\mathcal{E}} \right].
	\end{align*}
	On the other hand, the sum of diagonal elements is
	\begin{align*}
		\sum_{i=1}^n \En[G_i^2\, \ind{\mathcal{E}}] = n \En\left[\frac{K^2(X_1)}{\left(\sum_{i=1}^n K_h(X_i)\right)^2}
		\, \ind{\mathcal{E}}
		\right],
	\end{align*}
	which is precisely the variance term $\sigma_X^2$. Finally,
	\begin{align*}
		\sum_{i\neq j} \En[G_iG_j\, \ind{\mathcal{E}}] &= h^{2\beta} (BL_p p_{\min}^{-1} + L_f)^2 \En \left[ \frac{\sum_{i\neq j} K_h(X_i)K_h(X_j)}{\left(\sum_{i=1}^n K_h(X_i)\right)^2}\, \ind{\mathcal{E}}\right]\\
		&\leq h^{2\beta} (BL_p p_{\min}^{-1} + L_f)^2 \En \left[ \frac{\sum_{i,j=1}^n K_h(X_i)K_h(X_j)}{\left(\sum_{i=1}^n K_h(X_i)\right)^2}\, \ind{\mathcal{E}}\right]\\
		&\leq h^{2\beta} (BL_p p_{\min}^{-1} + L_f)^2.
	\end{align*}
	
\end{proof}

\subsection{Proofs of Theorem~\ref{thm:beta01} and ~\ref{thm:beta12}}

The two theorems follow immediately from Lemmas~\ref{lem:variance}, \ref{lem:bias_holder_01}, and \ref{lem:bias_holder_12} by balancing
$$n\exp\left\{-Cnh^d\right\}+\frac{C}{nh^d}+C h^{2\beta}$$
with $h=n^{-\frac{1}{2\beta+d}}$.

\section{Discussion}

We presented a proof of concept: an interpolating rule can achieve optimal rates for the problems of nonparametric estimation and prediction with square loss. Our proof technique extends to other kernels where the indicator over the unit Euclidean ball in \eqref{eq:def_our_kernel} is replaced with a function that dominates an appropriately scaled indicator. The analysis also works for non-singular kernels under the assumption of square integrability (required only in Eq.~\eqref{eq:square_integrability}). 

While each pair $(X_i,Y_i)$ is fit exactly by the proposed estimator, the influence of the datapoint is local. In aggregate, however, the function $f_n$ is being ``pulled'' towards the true regression function $f$. Whether a similar phenomenon occurs in other interpolating rules---such as overparametrized neural networks---requires further investigation.

\bibliography{refs}

\end{document}